\documentclass[twoside,11pt]{article}

\usepackage{vmargin}            % redéfinir les marges 
\setmarginsrb{3cm}{3cm}{3cm}{3cm}{1cm}{1cm}{1cm}{1cm} 
\usepackage{amsmath,amssymb,amsthm}
\newtheorem{thm}{Theorem}[section]

\newtheorem{lemma}[thm]{Lemma}

\newtheorem{dfn}{Definition}[section]

\def\1{\ensuremath{\mathrm{1}\hspace{-.35em} \mathrm{1}}} % indicatrice

\def\E{\mathbb{E}}

\def\R{\mathbb{R}}

\begin{document}

\title{Optimal learning with Bernstein Online Aggregation}

\author{ Olivier Wintenberger\\ {\tt olivier.wintenberger@upmc.fr}\\
 Sorbonne Universit\'es, UPMC Univ Paris 06\\
       LSTA, Case 158,
4 place Jussieu\\
75005 Paris,
FRANCE\\
\&\\
University of Copenhagen, DENMARK 
}
\date{}
%\editor{? }

\maketitle

\begin{abstract}
We introduce a new recursive aggregation procedure called Bernstein Online Aggregation (BOA). Its exponential weights include a second order refinement. The procedure is optimal for the model selection aggregation problem in the bounded iid setting for the square loss:  the excess of risk of its batch version achieves the fast rate of convergence $\log(M)/n$ in deviation. The BOA procedure is the first online algorithm that satisfies this optimal fast rate. The second order refinement  is required to achieve the optimality in deviation as the classical exponential weights cannot be optimal, see \cite{audibert:2009}. This refinement is settled thanks to a new stochastic conversion that estimates the cumulative predictive risk in any stochastic environment with observable second order terms. The observable second order term is shown to be sufficiently small to assert the fast rate in the iid setting when the loss is Lipschitz and strongly convex. We also introduce a multiple learning rates version of BOA. This fully adaptive BOA procedure is also optimal, up to a $\log\log(n)$ factor. \end{abstract}


\section{Introduction and main results}
We consider the online setting where observations $\mathcal D_t=\{(X_1,Y_1),\ldots,(X_t,Y_t)\}$ are available recursively ($(X_0,Y_0)=(x_0,y_0)$ arbitrary). The goal of statistical learning is to predict $Y_{t+1}\in\R$ given $X_{t+1}\in \mathcal X$, for $\mathcal X$ a probability space, on the basis of $\mathcal D_t$. In this paper, we index with the subscript $t$ any random element that is adapted to $\sigma(\mathcal D_t)$. A learner is a function $\mathcal X\mapsto \R$, denoted $ f_t$, that depends only on the past observations $\mathcal D_t$ and such that $ f_t(X_{t+1})$ is close to $Y_{t+1}$. This closeness at time $t+1$ is addressed by the predictive risk
$$
\E[\ell(Y_{t+1}, f_t(X_{t+1}))~|~\mathcal D_t]
$$ 
where $\ell:\R^2\to\R$ is a loss function. We define an online learner $  f$ as a recursive algorithm that produces at each time $t\ge 0$ a learner: $  f =( f_0, f_1, f_2,\ldots)$. The accuracy of an online learner is quantified by
the cumulative predictive risk
$$
R_{n+1}( f)=\sum_{t=1}^{n+1}\E[\ell(Y_{t}, f_{t-1}(X_{t}))~|~\mathcal D_{t-1}].
$$
We will motivate  the choice of this criteria later in the introduction.

Given a finite set $\mathcal H=\{f_1,\ldots,f_M\}$ of online learners such that $  f_j=( f_{j,t})_{t\ge 0}$, we aim at finding optimal online aggregation procedures $$\hat f=\Big(\sum_{j=1}^M \pi_{j,0}f_{j,0},\sum_{j=1}^M \pi_{j,1}f_{j,1},\sum_{j=1}^M \pi_{j,2}f_{j,2}\ldots\Big)$$
with $\sigma(\mathcal D_t)$-measurable weights $ \pi_{j,t}\ge 0$, $\sum_{j=1}^M  \pi_{j,t}=1$, $t=0,\ldots,n$. We call deterministic aggregation procedures $f_\pi$ any online learner of the form $$f_\pi=\Big(\sum_{j=1}^M\pi_jf_{j,0},\sum_{j=1}^M\pi_jf_{j,1},\sum_{j=1}^M\pi_jf_{j,2},\ldots\Big)$$ with $\pi=(\pi_j)_{1\le j\le M}$ with $\sum_{j=1}^M \pi_{j}=1$. Notice that $\pi$ can be viewed as a probability measure on the random index $J\in \{1,\ldots,M\}$. We will also use the notation $\pi_t$ for the probability measure $(\pi_{j,t})_{1\le j\le M}$ on $\{1,\ldots,M\}$. Then, $f_\pi=\E_\pi[f_J]$ and $\hat f =(\E_{\pi_0}[f_{J,0}],\E_{\pi_1}[f_{J,1}],\E_{\pi_2}[f_{J,2}],\ldots)$. The predictive performances of an online aggregation procedure $\hat f$ is compared with the best deterministic aggregation of $\mathcal H$ or the best element of $\mathcal H$. We refer to these two different objectives as, respectively, the convex aggregation Problem (C) or the model selection aggregation Problem (MS). The performance of online aggregation procedures is usually measured using the cumulative loss (in the context of individual sequences prediction, see the seminal book \cite{cesabianchi:lugosi:2006}). The first aim of this paper is to use instead the cumulative predictive risk for any stochastic process $(X_t,Y_t)$. However, to define properly the notion of optimality as in \cite{nemirovski:2000,tsybakov:2003}, we will also consider the specific iid setting of independent identically distributed observations $(X_t,Y_t)$ when the online learners are constants: $f_{j,t}=f_j$, $ t\ge 0$. In the iid setting, we suppress the indexation with time $t$ as much as possible and we define the usual risk $\bar R(f)= \E[\ell(Y,f(X))]=(n+1)^{-1}R_{n+1}(f)$ for any constant learner $f$. A batch learner is defined as $
\bar f=f_\pi$ for $\sigma(\mathcal D_n)$-measurable weights $\pi_j\ge 0$, $\sum_{j=1}^M\pi_j=1$. 
Lower bounds for the excesses of risk
$$
\mbox{Problem (C):}\qquad \bar R(\bar f)-\inf_\pi \bar R(f_\pi),\qquad\mbox{Problem (MS):}\qquad \bar R(\bar f)-\min_j \bar R(f_j),
$$
are provided for the expectation of the square loss in \cite{nemirovski:2000,tsybakov:2003}. Theorem 4.1 in \cite{rigollet:2012} provides sharper lower bounds in deviation for some strongly convex Lipschitz losses and deterministic $X_t$s. These lower bounds are called the optimal rates of convergence and a weaker version applies also in our more general setting with stochastic (but possibly degenerate) $X_t$s.  For Problem (C), we retain the rate $\sqrt{\log(M)/n}$ that is optimal when $M>\sqrt n$ and $\log(M)/n$ for Problem (MS), see \cite{rigollet:2012} for details.  We are now ready to define the notion of optimality:
\begin{dfn}[adapted from Theorem 4.1 in \cite{rigollet:2012}]\label{def}
In the iid setting, a batch aggregation procedure is optimal for Problems (C) when $M>\sqrt n$  or (MS) if there exists $C>0$ such that with probability $1-e^{-x}$, $x>0$, it holds
$$
\bar R(\bar f)-\inf_\pi \bar R(f_\pi)\le C\frac{\sqrt{\log M}+x}{\sqrt n},  \qquad\mbox{or}\qquad \bar R(\bar f)-\min_j \bar R(f_j)\le C\frac{\log M+x}n.
$$
\end{dfn}
Very few known procedures achieve the fast rate $\log(M)/n$ in deviation and none of them are issued from an  online procedure. In this article, we provide the Bernstein Online Aggregation  (BOA)  that is proved to be  the first  online aggregation procedure such that its batch version, defined as $\bar f=(n+1)^{-1}\sum_{t=0}^{n}\hat f_t$, is optimal. Before defining it properly, let us review the existing optimal procedures for Problem (MS).

The batch procedures in \cite{audibert:2007,lecue:mendelson:2009,lecue:rigollet:2013} achieve the optimal rate in deviation.  A priori, they face practical issues as they require a computational optimization technique to approximate the weights that are defined as an optimum. A step further has been done in the context of quadratic loss with gaussian noise  in \cite{dai:rigollet:zhang:2012} where an explicit iterative scheme is provided. We will now explain why the question of the existence of an online algorithm whose batch version achieves fast rate of convergence in deviations remained open (see the conclusion of \cite{audibert:2009}) before our work. Optimal (for the regret) online  aggregation procedures are exponential weights algorithms (EWAs), see  \cite{vovk:1990,haussler:kivinen:warmuth:1998}. The batch versions of EWAs  coincides with the Progressive Mixture Rules (PRMs). In the iid setting, the properties of the excess of risk of such procedures have been extensively studied in \cite{catoni:2004}. PRMs achieve the fast optimal rate $\log(M)/n$ in expectation  (that follows  from the expectation of the optimal regret bound by an application of Jensen's inequality, see \cite{catoni:2004,juditsky:rigollet:tsybakov:2008}). However, PRMs are suboptimal in deviation, i.e. the optimal rate cannot hold with high probability, see \cite{audibert:2007,dai:rigollet:zhang:2012}. It is because the optimality for the regret defined as in \cite{haussler:kivinen:warmuth:1998} does not coincides with the notion of optimality for the risk in deviation used in Definition \ref{def}.  

The optimal BOA procedure is obtained using a necessary second order refinement of EWA. Figure \ref{fig:1} describes the computation of the weights in the BOA procedure where $\ell_{j,t}$ denotes the opposite of the instantaneous regret $\ell(Y_{t}, f_{j,t-1}(X_{t})) -\E_{\pi_{t-1}}[\ell(Y_{t}, f_{J,t-1}(X_{t}))]$ (linearized when the loss $\ell$ is convex). 
\begin{figure}[h!]
 \fbox{
 \begin{minipage}[c]{15cm}
 {\it Parameters:} Learning rate $\eta>0$.\\
 {\it Initialization:} Set $\pi_{j,0}>0$ such that $\sum_{j=1}^M\pi_{j,0}=1$.\\
{\it For:} Each time round $ 1\le t\le n$,  compute the weight vector $\pi_{t} = (\pi_{j,t})_{1 \leq j \leq M} $:  
$$
\pi_{j,t }=\frac{\exp(-\eta\ell(Y_{t}, f_{j,t-1}(X_{t})) -\eta^2 \ell_{j,t}^2 )\pi_{j,t-1}}{\E_{\pi_{t-1}}[\exp(-\eta \ell(Y_{t},f_{J,t-1}(X_{t})) -\eta^2 \ell_{J,t}^2)]}=\frac{\exp(-\eta\ell_{j,t}(1+\eta \ell_{j,t}) )\pi_{j,t-1}}{\E_{\pi_{t-1}}[\exp(-\eta\ell_{J,t}(1+\eta \ell_{J,t}))]}.
$$
\end{minipage} 
 }
\caption{The BOA algorithm}
\label{fig:1}
\end{figure}
Other procedures already exist with different second order refinements, see \cite{audibert:2009,hazan:kale:2010}. None of them have been proved to be optimal for (MS) in deviation. The choice of the second order refinement is crucial. In this paper, the second order refinement is chosen as $\ell_{j,t}^2$ with
$$
\ell_{j,t}=\ell(Y_{t}, f_{j,t-1}(X_{t})) -\E_{\pi_{t-1}}[\ell(Y_{t}, f_{J,t-1}(X_{t}))].
$$
Notice that the second order refinement $\ell_{j,t}^2$ tends to stabilize the procedure as the distances between  the losses of the learners and the aggregation procedure are costly.

We achieve an upper bound for the excess of the cumulative predictive risk by  first deriving a second order bound on the regret: 
$$
\textnormal{Err}_{n+1}(\hat f)-\textnormal{Err}_{n+1}(f_\pi)\qquad \mbox{where}\qquad \textnormal{Err}_{n+1}(f)=\sum_{t=1}^{n+1}\ell(Y_{t},  f_{t-1}(X_{t})).
$$
Second we extend it to an upper bound on the excess of the cumulative predictive risk $R_{n+1}(\hat f)-  R_{n+1}(f_\pi)$ in any stochastic environment. In previous works, the  online to batch conversion follows  from an application of a Bernstein inequality for martingales. It provides a control of the deviations in the stochastic environment via the predictable quadratic variation, see for instance \cite{freedman:1975,zhang:2005,kakade:tewari:2008,gaillard:stoltz:vanerven:2014}. Here we prefer to use an empirical counterpart of the classical Bernstein inequality, based on the quadratic variation instead of the predictive quadratic variation. For any martingale $(M_t)$, we denote $\Delta M_t=M_t-M_{t-1}$ its difference ($\Delta M_0=0$ by convention) and $[M]_t=\sum_{j=1}^t\Delta M_j^2$ its quadratic variation. We will use the following new empirical Bernstein inequality:
\begin{thm}\label{th:mart}
Let $(M_t)$ be a martingale such that $\Delta M_t\ge -1/2$ a.s. for all $t\ge 0$.
Then for any $n\ge 0$ we have
$\E[\exp( M_n-   [M]_n)]\le 1.$ Without any boundedness assumption, we still have
\begin{equation}\label{eq:cmsu}
\E\left[\exp\left( 2^{-1}\left(M_n-   [M]_n-\sum_{t=1}^n\Delta M_t\1_{\Delta M_t< -1/2}\right)\right)\right]\le 1.
\end{equation}
\end{thm}
Empirical Bernstein's inequalities have already been developed in \cite{audibert:2006,maurer:pontil:2009}  and use in the multi-armed bandit and penalized ERM problems.  Applying Theorem \ref{th:mart}, we estimate successively  the deviations of two different martingales 
\begin{enumerate}
\item $\Delta M_{J,t}= -\eta \ell_{J,t}$ as a function of $J$ distributed conditionally as $\pi_{t-1}$ on $\{1,\ldots,M\}$,
\item $ M_{j,t}=\eta(R_t(\hat f)-  R_t(f_j)-\textnormal{Err}_t(\hat f)+\textnormal{Err}_t(f_j))$ such that $\Delta M_{j,t}=\eta(\E_{t-1}[\ell_{j,t}]-\ell_{j,t})$ where $\E_{t-1}$ denotes the expectation of $(X_t,Y_t)$ conditionally on $\mathcal D_{t-1}$, $1\le j\le M$.
\end{enumerate}
The first application 1. of Theorem \ref{th:mart} will provide a second order bound on the regret in the deterministic setting whereas the second application 2. will provide the new stochastic conversion. In both cases, the second order term will be equal to $\eta^{-1}[M_j]_{n+1}=\eta\sum_{t=1}^{n+1}\ell_{j,t}^2$ after renormalization. It is the main motivation of BOA; as our notion of  optimality requires a stochastic conversion, a second order term necessarily appears in the bound of the excess of the cumulative predictive risk. An online procedure will achieve good performances in the batch setting if it is regularized with  the necessary cost due to the stochastic conversion. The BOA procedure achieves this aim by incorporating this second order term in the computation of the weights.

In the first application 1. of Theorem \ref{th:mart}, we have $\E_{\pi_{t-1}}[\Delta M_{J,t}]=0$ and an application of Theorem \ref{th:mart} yields the regret bound of Theorem \ref{th:reg}:
$$
\E_{\hat \pi}[\textnormal{Err}_{n+1}( f_J)]\le  \inf_\pi \Big\{\E_\pi\Big[\textnormal{Err}_{n+1}(f_J)+\eta\sum_{t=1}^{n+1} \ell_{J,t}^2+ \frac{ \log(\pi_J/\pi_{J,0})}{\eta} \Big]\Big\},
$$
where $\E_{\hat \pi}[\textnormal{Err}_{n+1}( f_J)]=\sum_{t=1}^{n+1}\E_{\hat \pi_{t-1}}[\ell(Y_{t} f_{J,t}(X_{t}))]$. 
Such second order bounds also hold for the regret of other algorithms, see \cite{cesabianchi:mansour:stoltz:2007,gaillard:stoltz:vanerven:2014,luo:schapire:2015,koolen:vanerven:2015}. Using the new stochastic conversion based on the application 2. of Theorem \ref{th:mart}, this second order regret bound is converted to a one on the cumulative predictive risk (see Theorem \ref{th:regcr}). With probability $1-e^{-x}$, $x>0$, we have
$$
\E_{\hat \pi}[R_{n+1}( f_J)]\le  \inf_\pi \Big\{\E_\pi\Big[R_{n+1}(f_J)+2\eta\sum_{t=1}^{n+1} \ell_{J,t}^2+ \frac{\log(\pi_J/\pi_{J,0})+x}{\eta}\Big] \Big\}.
$$
Thanks to the use of the cumulative predictive risk, this bound is valid in any stochastic environment. We will extend it in various directions. We will introduce 
\begin{itemize}
\item the "gradient trick" to bound the excess of the cumulative predictive risk in Problem (C),
\item the multiple learning rates for adapting the procedure and
\item the batch version of BOA to achieve the fast rate of convergence in Problem (MS).
\end{itemize}

The "gradient trick" is a standard argument to solve Problem (C),  see \cite{cesabianchi:lugosi:2006}. When the loss $\ell$ is convex with respect to its second argument, its sub-gradient is denoted $\ell'$. In this case, we  consider a convex version of the BOA procedure described in Figure \ref{fig:1}. The original loss $\ell$ is replaced with its linearized version 
\begin{equation}\label{eq:lin}
\ell'(Y_{t},\hat f_{t-1}(X_{t}))f_{j,t-1}(X_{t}),
\end{equation}
and we denote, with some abuse of notation,
$$
\ell_{j,t}= \ell'(Y_{t},\hat f_{t-1}(X_{t}))(f_{j,t-1}(X_t)- \hat f_{t-1}(X_{t})).
$$
Linearizing the loss, we can compare the regret of the (sub-gradient version of the) BOA procedure $\hat f=\E_{\hat\pi}[f_J]$ with the best deterministic aggregation of the elements in the dictionary. We obtain in Theorem \ref{th:reg} a second order regret bound   for Problem (C)
$$
\textnormal{Err}_{n+1}(\hat f)\le  \inf_{\pi} \Big\{\textnormal{Err}_{n+1}(f_\pi)+\eta\sum_{t=1}^{n+1} \E_\pi[ \ell_{J,t}^2]+ \frac{\E_\pi[\log(\pi_J/\pi_{J,0})]}\eta\Big\}.
$$
When trying to optimize the regret bound in the learning rate $\eta>0$, we obtain
$$
\sqrt{\frac{\E_\pi[\log(\pi_J/\pi_{J,0})]}{ \sum_{t=1}^{n+1} \E_\pi[ \ell_{J,t}^2]}}\le \E_\pi\left[\sqrt{\frac{\log(\pi_{J,0}^{-1})}{ \sum_{t=1}^{n+1}  \ell_{J,t}^2}}\right],
$$
As $\pi$ is unknown, this tuning parameter is not tractable in practice. Its worst case version $ \max_j \sqrt{\log(\pi_{j,0}^{-1})}/\sqrt{\sum_{t=1}^{n+1}  \ell_{j,t}^2}$ is not satisfactory.  
Multiple learning rates have been introduced by \cite{blum:mansour:2005} to solve this issue (see also \cite{gaillard:stoltz:vanerven:2014}). We introduce the multiple learning rates version of BOA  in Figure \ref{fig:2} which is also fully adaptative as it adapts to any possible range of observations. 

\begin{figure}[h!]
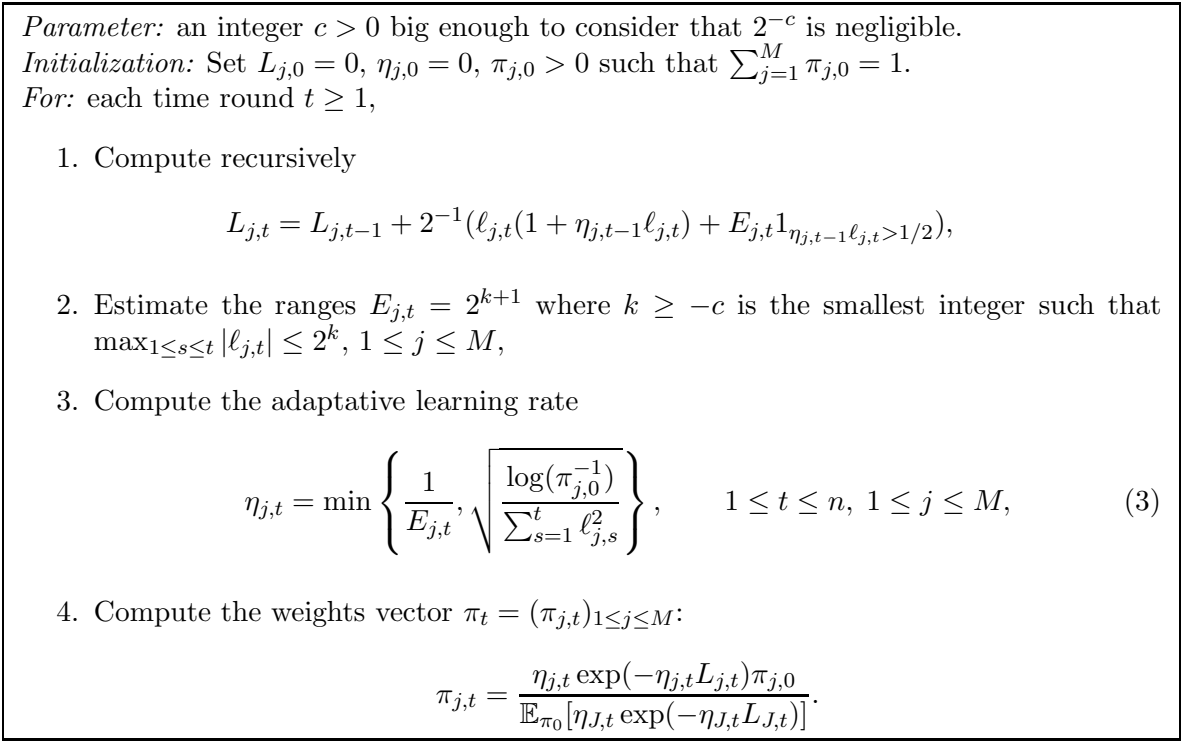

 \fbox{
 \begin{minipage}[c]{15cm}
 {\it Parameter:} an integer $c>0$   big enough to consider that $2^{-c}$ is negligible. \\
 {\it Initialization:} Set $L_{j,0}=0$, $\eta_{j,0}=0$, $\pi_{j,0}>0$ such that $\sum_{j=1}^M\pi_{j,0}=1$.\\
{\it For:} each time round $t \geq 1$,  
 \begin{enumerate} 
\item Compute recursively $$\hspace{-1cm}L_{j,t}=L_{j,t-1}+2^{-1}(\ell_{j,t}(1 +\eta_{j,t-1}\ell_{j,t})+E_{j,t}1_{\eta_{j,t-1}\ell_{j,t}>1/2}), $$
\item Estimate the ranges  $E_{j,t}=2^{k+1}$ where $k\ge -c$ is the smallest integer such that 
$
\max_{1\le s\le t}|\ell_{j,t}|\le 2^k
$, $1\le j\le M$,
\item Compute the adaptative learning rate
\begin{equation}\label{eq:lrEt}
\eta_{j,t}=\min\left\{\frac1{E_{j,t}},\sqrt{\frac{\log (\pi_{j,0}^{-1})}{\sum_{s=1}^t \ell_{j,s}^2}}\right\},  \qquad 1\le t\le n,\; 1\le j\le M,
\end{equation}
\item  Compute the weights vector $\pi_{t} = (\pi_{j,t})_{1 \leq j \leq M} $:  
$$
\pi_{j,t }=
\frac{\eta_{j,t}\exp(- \eta_{j,t} L_{j,t} )\pi_{j,0}}{ \E_{\pi_0}[\eta_{J,t}\exp(- \eta_{J,t} L_{J,t})]}.
$$
\end{enumerate}
\end{minipage} 
 }
\caption{The fully adaptive BOA procedure}
\label{fig:2}
\end{figure}
The novelty, compared with the "doubling trick" developed in \cite{cesabianchi:mansour:stoltz:2007}, is the dependence of the learning rates and the estimated ranges with respect to $j$ and the expression of the weights with respect to the learning rates and the ranges estimators.  For this adaptive BOA procedure we obtain regret bounds such as
$$
\textnormal{Err}_{n+1}(\hat f)\le  \inf_{\pi} \left\{\textnormal{Err}_{n+1}(f_\pi)+C\E_\pi \left[\sqrt{\sum_{t=1}^{n+1} \ell_{J,t}^2\log (\pi_{J,0}^{-1})}+ E_J\log (\pi_{J,0}^{-1})\right]\right\},
$$
for some "constant"  $C>0$ that grows as $\log\log(n)$, see Theorem \ref{th:regad}   for details. Such second order bounds involving excess losses terms as the $\ell_{j,t}$s have been proved for other algorithms in \cite{gaillard:stoltz:vanerven:2014,luo:schapire:2015,koolen:vanerven:2015}. We refer to these articles for nice consequences of such bounds in the individual sequences framework. Here again, the stochastic conversion holds in any stochastic environment and we obtain with probability $1-e^{-x}$, $x>0$
$$
R_{n+1}(\hat f)\le  \inf_{\pi} \left\{R_{n+1}(f_\pi)+C\E_\pi \left[\sqrt{\sum_{t=1}^{n+1} \ell_{J,t}^2(\log (\pi_{J,0}^{-1})+x)}+E_J\log (\pi_{J,0}^{-1})+x)\right]\right\},
$$
The optimal bound for Problem (C) is proved in the very general setting: under a boundedness assumption, the rate of convergence of $R_{n+1}(\hat f)/(n+1)$ (an upper bound of $\bar R(\bar f)$ in the iid setting), is smaller than $\sqrt{\log(M)/n}.$ 
Notice that  the classical online-to-batch conversion leads to similar results, see \cite{cesabianchi:mansour:stoltz:2007,gerchinovitz:2013,gaillard:stoltz:vanerven:2014}, and that our second order refinement is not necessary to obtain such upper bounds. Notice also that such upper bounds  were already derived in non iid settings  for the excess of risk (and not the cumulative predictive risk) under  restrictive dependent assumptions, see \cite{alquier:li:wintenberger:2013,mohri:rostamizadeh:2010,agarwal:duchi:2013}.  It is remarkable to extend the optimal bound of Problem (C) to any stochastic environment thanks to the use of the cumulative predictive risk. It is because the cumulative predictive risk of $f_\pi$ in the upper bound takes into account the dependence as it is a random variable in non iid settings. We believe that the cumulative predictive risk is the correct criteria to assert the prediction accuracy of online algorithms in stochastic environment as it  coincides with the regret in the deterministic setting and with the classical risk for batch procedures in the iid setting. Moreover, it appears naturally when using the minimax theory approach, see \cite{abenethy:2009}. However, up to our knowledge, it is the first time that the cumulative predictive risk is used to compare an online procedure with deterministic aggregation procedures.

The fast rate of convergence $\log(M)/n$ in Problem (MS) is achieved thanks to a careful study of the second order terms $\sum_{t=1}^{n+1} \ell_{j,t}^2$. It also requires more conditions on the loss in order  to behave locally like the square loss, see \cite{audibert:2009}. We restrict us to losses $\ell$ that are $C_\ell$-strongly convex and $C_b$-Lipschitz functions in the iid setting, see \cite{kakade:tewari:2008} for an extensive study of this context. We fix the initial weights uniformly $\pi_{j,0}=M^{-1}$. We obtain in Theorem \ref{th:opt} the fast rate of convergence for the batch version of the  BOA procedure; with probability $1-e^{-x}$, $x>0$, 
$$
\bar R(\bar f)\le \frac{R_{n+1}(\hat f)}{n+1}\le \min_{1\le j\le M}R(f_j)+\frac{2 \log(M)+3x}{\eta(n+1)},
$$
for $\eta^{-1}$ larger than $C_b^2/C_\ell$ up to a multiplicative constant. The second order term is bounded by the excess of risk using  the strong convexity assumption on the loss. We conclude by providing the fast rate bound on the excess of risk of the batch version of the adaptive BOA at the price of larger "constants"  that grows at the rate $\log\log(n)$.

The paper is organized as follows: We present the second order regret bounds for different versions of BOA in Section \ref{sec:sob}. The new stochastic conversion and the excess of cumulative predictive risk bounds in a stochastic environment are provided in Section \ref{sec:se}. In the next Section, we introduce some useful probabilistic preliminaries.

\section{Preliminaries}\label{sec:tm}
As in \cite{audibert:2009}, the recursive argument for supermartingales will be at the core of the proofs developed in this paper. It will be used jointly with the variational form of the entropy to provide second order regret bounds.

\subsection{The proof of the martingale inequality in Theorem \ref{th:mart}}
The proof of the first empirical Bernstein inequality for martingales of Theorem \ref{th:mart} follows from an exponential inequality and by a classical recursive supermartingales argument, see \cite{freedman:1975}. As $X=\Delta M_t\ge -1/2$ a.s., from the inequality $\log(1+x)\ge x-x^2$ for $x>-1/2$ (stated as Lemma 1 in \cite{cesabianchi:mansour:stoltz:2007}), we have 
\begin{equation}\label{eq:cms}
X-X^2\le\log(1+X) \Leftrightarrow \exp(X-X^2)\le 1+X\Rightarrow\E_{t-1}[\exp(X-X^2)]\le 1.
\end{equation}
Here we used that $\E_{t-1}[X]=0$ as $X=\Delta M_t$ is a difference of martingale. The proof ends by using the classical recursive argument for supermartingales; from the definition of the difference of martingale $X=\Delta M_t$, we obtain as a consequence of \eqref{eq:cms} that 
$$
\E[\exp(M_t-[M]^2_t)]\le \E[\exp(M_{t-1}-[M]^2_{t-1})].
$$
As $\E[\exp(M_0-[M]^2_0)]=1$, applying a recursion for $t=1,\ldots, n$ provides the desired result.\\

Without any boundedness assumption, we have 
\begin{align*}
\E[\exp(X-X^2)]&=\E[\exp(X-X^2)\1_{X\ge -1/2}]+\E[\exp(X-X^2)\1_{X< -1/2}]\\
&\le \E[(1+X)\1_{X\ge -1/2}]+\E[\1_{X< -1/2}]\le 1+\E[X\1_{X\ge -1/2}].
\end{align*}
As $X$ is centered, we can bound 
\[
1+\E[X\1_{X\ge -1/2}]=1-\E[X\1_{X< -1/2}]\le \E[1+X\1_{X< -1/2}]\le\E[\exp( X\1_{X< -1/2})].
\]
Using Cauchy-Schwarz inequality and the preceding arguments, we obtain
\begin{align*}
\E[\exp( 2^{-1}(X-  X^2-X\1_{X< -1/2}))]&\le \sqrt{\E[\exp( X-  X^2)]\E[\exp(-X\1_{X< -1/2})]}\\
&\le \sqrt{\E[\exp( X\1_{X< -1/2})]\E[\exp(-X\1_{X< -1/2})]}.
\end{align*}
The desired result follows from the Jensen's inequality followed by the same recursive argument for supermartingales as above.

\subsection{The variational form of the entropy}
The relative entropy (or Kullback-Leibler divergence) $\mathcal K(Q,P)=\E_Q[\log(dQ/dP)]$ is a pseudo-distance between any probability measures $P$ and $Q$. Let us remind the basic property of the entropy: the  variational formula of the entropy originally proved in full generality in \cite{donsker:varadhan:1975}. We consider here a version well adapted for obtaining second order regret bounds:
\begin{lemma}
\label{le:vf}
For any probability measure $P$  on $\mathcal X$ and any measurable  functions $h$, $g:\mathcal X\rightarrow\mathbb{R}$ we have:
\begin{multline} \label{eq:vf}
\E_P[\exp(h-\E_P[h]-g)]\le 1\\
\Longleftrightarrow \E_Q[h]- \E_P[h]\le  \E_Q[g]+\mathcal K(Q,P),\qquad \mbox{for any probability measure }Q.
\end{multline}
The left hand side corresponds to the right hand side with $Q$ equals the Gibbs measure $\E_P[e^{h-g}]dQ =  e^{h-g } 
dP $.
\end{lemma}
That the Gibbs measure realizes the dual identity is at the core of the PAC-bayesian approach. Exponential weights aggregation procedures arise naturally as they can be considered as Gibbs measures, see \cite{catoni:2007}.

\section{Second order regret bounds for the BOA procedure}\label{sec:sob}
\subsection{First regret bounds and link with the individual sequences framework}
We work conditionally on $\mathcal D_{n+1}$; it is the deterministic setting, similar than in \cite{gerchinovitz:2013}, where $(X_t,Y_t)=(x_t,y_t)$ are provided recursively for $1\le t\le n$. In that case, the cumulative loss $\textnormal{Err}_{n+1}(f)$ quantify the prediction of $f=(f_0,f_1,f_2,\ldots)$. We state first a regret bound for non convex losses, and then move to the case of convex losses combined with the "gradient trick" as in the Appendix of \cite{gaillard:stoltz:vanerven:2014}. Recall that $\E_{\hat \pi}[\textnormal{Err}_{n+1}( f_J)]=\sum_{t=1}^{n+1}\E_{\hat \pi_{t-1}}[\ell(Y_{t} ,f_{J,t}(X_{t}))]$.
\begin{thm}\label{th:reg}
Assume that $\eta>0 $ satisfies
\begin{equation}\label{eq:eta}
\eta\max_{ 1\le t\le n+1} \max_{1\le j\le M} {\ell_{j,t}}\le 1/2,
\end{equation}
then the cumulative loss of the BOA  procedure with $
\ell_{j,t}= \ell(Y_{t},f_{j,t-1}(X_t))-\ell(Y_{t},\hat f_{t-1}(X_{t}))
$ satisfies
$$
\E_{\hat \pi}[\textnormal{Err}_{n+1}( f_J)]\le  \inf_{\pi} \Big\{\E_\pi\Big[\textnormal{Err}_{n+1}(f_J)+\eta\sum_{t=1}^{n+1} \ell_{J,t}^2\Big]+ \frac{\mathcal K(\pi,\pi_0)}{\eta} \Big\}.
$$
If $\ell$ is convex with respect to its second argument,  the cumulative loss of the BOA  procedure with  $
\ell_{j,t}= \ell'(Y_{t},\hat f_{t-1}(X_{t}))(f_{j,t-1}(X_t)- \hat f_{t-1}(X_{t}))
$ also satisfies
$$
\textnormal{Err}_{n+1}(\hat f)\le  \inf_{\pi} \Big\{\textnormal{Err}_{n+1}(f_\pi)+\eta\sum_{t=1}^{n+1} \E_\pi[ \ell_{J,t}^2]+ \frac{\mathcal K(\pi,\pi_0)}\eta\Big\}.
$$
\end{thm}
\begin{proof}
We consider  $\Delta M_{J,t+1}=-\eta \ell_{J,t+1}$ that is a centered random variable on $\{1,\ldots,M\}$ when $J$ is distributed as $\pi_{t}$.  Under the assumption \eqref{eq:eta}, $\Delta {M_{J,t+1}}\ge-1/2$ for any $0\le t\le n$ a.s.. An application of the inequality \eqref{eq:cms}  provides the inequality
\begin{equation}\label{eq:si}
\E_{\pi_{t}}[\exp(-\eta  \ell_{J,t+1}(1+\eta \ell_{J,t+1}))]\le 1.
\end{equation}
From the recursive definition of the BOA procedure provided in Figure \ref{fig:1}, we have the expression
$$
\pi_{j,t}=\frac{\exp (-\eta\sum_{s=1}^{t} \ell_{j,s}(1+  \eta \ell_{j,s}))\pi_{j,0}}{\E_{\pi_0}[\exp (-\eta\sum_{s=1}^{t} \ell_{J,s}(1+  \eta  \ell_{J,s}))]}.
$$
Plugging the expression of the weights $\pi_{j,t}$ in the inequality \eqref{eq:si} provides
$$
\E_{\pi_0}\Big[\exp\Big(-\eta\sum_{s=1}^{t+1} \ell_{J,s}(1+  \eta \ell_{J,s})\Big)\Big]\le \E_{\pi_0}\Big[\exp \Big(-\eta\sum_{s=1}^{t} \ell_{J,s}(1+  \eta \ell_{J,s})\Big)\Big].
$$
By a recursive argument on $0\le t\le n$ we obtain
$$
\E_{\pi_0}\Big[\exp\Big({-\eta\sum_{t=1}^{n+1} \ell_{J,t}(1+\eta\ell_{J,t})}\Big)\Big]\le1.
$$
Equivalently, using the variational form of the entropy \eqref{eq:vf}, 
\begin{equation}\label{eq:vfl}
0\le   \inf_\pi\Big\{\E_{\pi}\Big[\eta \sum_{t=1}^{n+1}\ell_{J,t}+ \eta^2 \sum_{t=1}^{n+1} \ell_{J,t}^2\Big]+  \mathcal K(\pi,\pi_0)\Big\},
\end{equation}
$\pi$ denoting  any probability measure on $\{1,\ldots,M\}$.
The first regret bound in Theorem \ref{th:reg} follows from the identity $\sum_{t=1}^{n+1}\ell_{J,t}=\textnormal{Err}_{n+1}(f_J)-\E_{\hat \pi}[\textnormal{Err}_{n+1}(f_J)]$. The second result follows by an application of the "gradient trick", i.e. noticing that 
\begin{align*}
\textnormal{Err}_{n+1}( \hat f)-\textnormal{Err}_{n+1}(f_\pi)&= \sum_{t=1}^{n+1}   \ell(Y_{t},\hat f_{t-1}(X_{t}))- \ell(Y_{t},\E_\pi[f_{J,t-1}](X_{t}))\\&\le\sum_{t=1}^{n+1}\ell'(Y_{t},\hat f_{t-1}(X_{t}))( \hat f_{t-1}(X_{t})-\E_\pi[f_{J,t-1}](X_{t}))\\&=\E_\pi\Big[\sum_{t=1}^{n+1}\ell'(Y_{t},\hat f_{t-1}(X_{t}))( \hat f_{t-1}(X_{t})-f_{J,t-1}(X_{t}))\Big]\\
&=-\E_\pi\Big[\sum_{t=1}^{n+1} \ell_{J,t}\Big].\end{align*}
\end{proof}
The second order term in the last regret bound is equal to
$$
\sum_{t=1}^{n+1} \E_\pi[\eta\ell_{J,t}^2]= \sum_{t=1}^{n+1} \E_{\pi}[\eta \ell'(Y_{t},\hat f_{t-1}(X_{t}) )^2(\hat f_{t-1}(X_{t}) - f_{J,t-1}(X_{t}) )^2].
$$
This term can be small because the sub-gradients are small or because the BOA weights are close to the objective $\pi$. Such second order upper bounds can heavily depend on the behaviors of the different learners $f_j$. Thus, a unique learning rate cannot be efficient in cases where the learners have different second order properties. To solve this issue, we consider the multiple learning rates  version of BOA 
\begin{equation}
\label{eq:ml}
\pi_{j,t }=\frac{\exp(-\eta_{j}\ell_{j,t}(1+\eta_{j} \ell_{j,t}) )\pi_{j,t-1}}{\E_{\pi_{t-1}}[\exp(-\eta_{J}\ell_{J,t}(1+\eta_{J} \ell_{J,t}))]}.
\end{equation} We can extend the preceding regret bound to this more sophisticated procedure:
\begin{thm}\label{th:ml}
Consider a loss  $\ell$ convex with respect to its second argument and multiple learning rates $\eta_j$, $1\le j\le M$, that are positive. If 
$$
\max_{ 1\le t\le n+1} \max_{1\le j\le M} {\eta_j\ell_{j,t}}\le 1/2,\qquad a.s.,
$$
then  the cumulative loss of the BOA procedure with multiple learning rates  satisfies
$$
\textnormal{Err}_{n+1}(\hat f)\le  \inf_{\pi} \Big\{\textnormal{Err}_{n+1}(f_\pi)+ \E_\pi\Big[ \eta_J\sum_{t=1}^{n+1}\ell_{J,t}^2+ \frac{\log(\pi_J/\pi_{J,0})+\log(\E_{\pi_0}[\eta_J^{-1}]/\E_{\pi}[\eta_J^{-1}])}{\eta_J}\Big]\Big\}.
$$
\end{thm}
\begin{proof}
Let us consider the weights $ \pi_{i,t}'=\eta_i^{-1}\pi_{i,t}/\E_{\pi_t}[\eta_j^{-1}]$, for all $1\le i\le M$ and $0\le t\le n+1$. Then, for any function $j\to h_j$ measurable on $\{1,\ldots,M\}$ we have the relation
\begin{equation}\label{eq:w'}
\E_{\pi_t'}[\eta_Jh_J]=\E_{\pi_t}[h_J]/\E_{\pi_t}[\eta_J^{-1}],\qquad 1\le t\le n+1.
\end{equation}
Consider $\Delta M_{j,t}=-\eta_j \ell_{j,t+1}$, $1\le j\le M$. Thanks to the identity \eqref{eq:w'}, $\Delta M_{J,t}$ is a centered random variable when $J$ is distributed as $ \pi_{t}'$ on $\{1,\ldots,M\}$. Moreover, the weights $ (\pi_{t}')$ satisfy the recursive relation \eqref{eq:ml}. Thus, one can apply the same reasoning than in the proof of Theorem \ref{th:reg}. We obtain an equivalent of the inequality \eqref{eq:vfl} 
$$
0\le   \inf_{ \pi'}\Big\{\E_{ \pi'}\Big[\eta_J \sum_{t=1}^{n+1}\ell_{J,t}+ \eta^2_J \sum_{t=1}^{n+1} \ell_{J,t}^2\Big]+  \mathcal K( \pi', \pi_0')\Big\},
$$
for $ \pi'$ denoting any probability measure on $\{1,\ldots,M\}$. Using the identity \eqref{eq:w'} to define $\pi$ from $\pi'$, and multiplying the above inequality with $\E_{\pi}[\eta_j^{-1}]>0$, we obtain$$
0\le   \inf_{ \pi}\Big\{\E_{ \pi}\Big[ \sum_{t=1}^{n+1}\ell_{J,t}+ \eta_J \sum_{t=1}^{n+1} \ell_{J,t}^2+\frac{\log( \pi'_J/\pi_{J,0}')}{\eta_J}\Big]\Big\}.
$$
The proof ends by identifying $\log( \pi'_j/\pi_{j,0}')$ and using the  "gradient trick" as in the proof of Theorem \ref{th:reg}.
\end{proof}
Notice that a simple corollary of the proof above is the simplified upper bound
$$
\textnormal{Err}_{n+1}(\hat f)\le  \min_{\pi} \Big\{\textnormal{Err}_{n+1}(f_\pi)+ \E_\pi\Big[ \eta_J\sum_{t=1}^{n+1}\ell_{J,t+1}^2+ \frac{\log(1/\pi_{J,0}')}{\eta_J}\Big]\Big\},
$$
where $\pi'_{j,0} =\eta_j^{-1}\pi_{j,t}/\E_{\pi_t}[\eta_J^{-1}]$. The initial weights $\pi_{j,0}$ are modified and the upper bound favors the learners with small learning rates $\eta_j$. It constitutes a drawback of the multiple learning rates version of BOA as we will see that  small learning rates will be associated with bad experts.  One can solve this issue by choosing the initial weights differently than classically. For example, with no information on the learners $f_j$, the initial weights can be chosen equal to
$$
\pi_{j,0}=\frac{\eta_j}{\sum_{j=1}^M\eta_j}\qquad \mbox{such that}\qquad \pi'_{j,0}=\frac1M,\qquad 1\le j\le M.
$$
In this case, $\log(1/\pi_{j,0}')\le \log(M)$ and the weights have the expression
$$
\pi_{j,t}=\frac{\eta_j\exp(-\eta_j\sum_{s=1}^t\ell_{j,s}(1+\eta_j\ell_{j,s})\pi_{j,0}}{\E_{\pi_0}[\eta_J\exp(-\eta_J\sum_{s=1}^t\ell_{J,s}(1+\eta_J\ell_{J,s})]},\qquad 1\le j\le M.
$$
The form of the weights becomes similar than the one of the adaptive BOA introduced in Figure \ref{fig:2} and studied in the next section. The second order regret bounds becomes 
$$
\textnormal{Err}_{n+1}(\hat f)\le  \min_{\pi} \left\{\textnormal{Err}_{n+1}(f_\pi)+ 2\sqrt{\log(M)}\E_\pi\left[ \sqrt{\sum_{t=1}^{n+1}\ell_{J,t}^2}\right]\right\},
$$
for learning rates tuned optimally
$$
\eta_j=\sqrt{\frac{\log(M)}{\sum_{t=1}^{n+1}\ell_{j,t}^2}},\qquad 1\le j\le M.
$$
However, the resulting procedure is not recursive because $\eta_j$ is $\sigma(D_{n+1})$ measurable. Such non recursive strategies are not convertible to the batch setting. 

Second order regret bounds similar to the one of Theorem \ref{th:ml} have been obtained in \cite{gaillard:stoltz:vanerven:2014,luo:schapire:2015,koolen:vanerven:2015} in the context of individual sequences.
In this context, we consider that  $Y_t=y_t$ for a deterministic sequence $y_0,\ldots,y_n$ ($(X_t)$ is useless in this context), see  \cite{cesabianchi:lugosi:2006} for an extensive treatment of that setting. We have $\mathcal D_t=\{y_0,\ldots,y_t\}$, $0\le t\le n$, and the online learners $f_j=(y_{j,1},y_{j,2},y_{j,3},\ldots)$ of the dictionary are called the experts. The cumulative loss is $\textnormal{Err}_{n+1}(\hat f)=\sum_{t=1}^{n+1}\ell(y_{t},\hat y_t)$ for any aggregative strategy  $\hat y_t=\hat f_{t-1}=\sum_{j=1}^M\pi_{j,t-1}y_{j,t}$ where $\pi_{j,t-1}$ are measurable functions of the past $\{y_0,\ldots,y_{t-1}\}$. We will compare our second order regret bounds to the ones of other  adaptive procedures from the individual sequences setting at the end of the next Section.

\subsection{A new adaptive method for exponential weights}

We described in Figure \ref{fig:2} the adaptive version of the BOA procedure. Notice that the adaptive version of the exponential weights
$$
\pi_{j,t }=
\frac{\eta_{j,t}\exp(- \eta_{j,t} L_{j,t} )\pi_{j,0}}{ \E_{\pi_0}[\eta_{J,t} \exp(- \eta_{J,t} L_{J,t})]},
$$ 
is different from \cite{cesabianchi:mansour:stoltz:2007} as the multiple learning rates $\eta_{j,t}$ depend on $j$. Moreover, the multiple learning rates appear in the exponential and as a multiplicative factor to solve the issue concerning the modification of the initial weights described above. Adaptive procedures of such form have been studied in \cite{gaillard:stoltz:vanerven:2014}. Another possibility consists in putting a prior on learning rates as in \cite{koolen:vanerven:2015}. Multiple  learning rates versions can be investigated for other exponential weights procedures than BOA. Notice also that for the first time we apply a multiple version of the "doubling trick" of \cite{cesabianchi:mansour:stoltz:2007}; the ranges estimators $E_{j,t}$ depend on $j$ and also appear in the exponential weights as a penalization when the ranges estimators are exceeded. 
We obtain a second order regret bound for the BOA procedure similar to the second order regret bounds obtained in Corollary 4 of \cite{gaillard:stoltz:vanerven:2014}:

\begin{thm}\label{th:regad}
Assume that $\ell$ is convex with respect to its second argument  and that $E_j$ defined by 
\[
\max_{1\le t\le n+1}|\ell_{j,t}|\le E_j,\qquad 1\le j\le M,
\]
satisfies $2^{-c}\le E_j\le E$ for all $1\le j\le M$. We have
\begin{multline*}
\textnormal{Err}_{n+1}(\hat f)\le  \inf_{\pi} \left\{\textnormal{Err}_{n+1}(f_\pi)+2\E_\pi \left[\sqrt{\displaystyle \sum_{t=1}^{n+1}  \ell_{J,t}^2}\left(\frac{\sqrt{2}}{\sqrt 2-1}\sqrt{ \log (\pi_{J,0}^{-1})}+\frac{B_{n,E}}{\sqrt{\log(\pi_{J,0}^{-1})}}\right)\right.\right.\\ +E_J(4(\log(\pi_{J,0}^{-1})+B_{n,E})+9)]\},
\end{multline*}
where $B_{n,E}= \log(1 +2^{-1}\log(n)+\log(E)+c\log(2))$ for all $n\ge 1$.
\end{thm}

\begin{proof}
We adapt the reasoning of the proof of Theorem \ref{th:ml} for  learning rates depending on $t$. Thus, the recursive argument holds only approximatively. 
For any $1\le t\le n$, let us consider the weights  $ \pi_t'$ as 
$$
\pi_{j,t}'=\frac{\eta_{j,t}^{-1}\pi_{j,t}}{\E_{\pi_t}[\eta_{J,t}^{-1}]}.
$$
We consider $\Delta M_{J,t+1}=-\eta_{J,t} \ell_{J,t+1}$ a centered random variable when $J$ is distributed as $ \pi_{t}'$ on $\{1,\ldots,M\}$. As  $\Delta M_{j,t+1}\ge -\eta_{j,t}E_{j,t+1}$, $j=1,\ldots,M$, we apply the inequality \eqref{eq:cmsu}:
$$
\E_{\pi'_t} [\exp (- \eta_{J,t}2^{-1}(\ell_{J,t+1}(1+ \eta_{J,t}\ell_{J,t+1}) +E_{J,t+1}1_{\eta_{J,t} \ell_{J,t+1}>1/2})) ]\le 1.
$$
By definition of the weights $\pi_t'$ and $\pi_t$, we have
\[
\pi_{j,t}'=\frac{ \exp(-\eta_{j,t}2^{-1}\sum_{s=1}^{t}(\ell_{j,s}(1+ \eta_{j,s-1}\ell_{j,s})+E_{j,s}1_{\eta_{j,s-1} \ell_{j,s}>1/2}))\pi_{j,0}}{\E_{\pi_0}[ \exp(-\eta_{J,t}2^{-1}\sum_{s=1}^{t}(\ell_{J,s}(1+ \eta_{J,s-1}\ell_{J,s})+E_{J,s}1_{\eta_{J,s-1} \ell_{J,s}>1/2}))]},\qquad1\le t\le n.
\]
Using the expression of the weights in the exponential inequality provides
\begin{multline}\label{eq:rec1}
\E_{\pi_0}\Big[  \exp\Big(-\eta_{J,t}2^{-1}\sum_{s=1}^{t+1}(\ell_{J,s}(1+ \eta_{J,s-1}\ell_{J,s})+E_{J,s}1_{\eta_{J,s-1} \ell_{J,s}>1/2})\Big)\Big]\\
\le \E_{\pi_0}\Big[\exp\Big(-\eta_{J,t}2^{-1}\sum_{s=1}^{t}(\ell_{J,s}(1+ \eta_{J,s-1}\ell_{J,s})+E_{J,s}1_{\eta_{J,s-1} \ell_{J,s}>1/2})\Big)\Big].
\end{multline}
Using the basic inequality $x\le  \alpha^{-1}x^\alpha + \alpha^{-1}(\alpha-1)\le x^\alpha + \alpha^{-1}(\alpha-1)$ for 
\[
x=  \exp\Big(-\eta_{j,t}2^{-1}\sum_{s=1}^{t}(\ell_{j,s}(1+ \eta_{j,s-1}\ell_{j,s})+E_{j,s}1_{\eta_{j,s-1} \ell_{j,s}>1/2})\Big)\ge 0
\] 
and $\alpha=\eta_{j,t-1}/\eta_{j,t}\ge 1$, we obtain for all $2\le t\le n$
\begin{multline}\label{eq:recurs}
\E_{\pi_0}\Big[ \exp\Big(-\eta_{J,t}2^{-1}\sum_{s=1}^{t}(\ell_{J,s}(1+ \eta_{J,s-1}\ell_{J,s})+E_{J,s}1_{\eta_{J,s-1} \ell_{J,s}>1/2})\Big)\Big]\\
\le \E_{\pi_0}\Big[ \exp\Big(-\eta_{J,t-1}2^{-1}\sum_{s=1}^{t}(\ell_{J,s}(1+ \eta_{J,s-1}\ell_{J,s})+E_{J,s}1_{\eta_{J,s-1} \ell_{J,s}>1/2})\Big)\Big]+\E_{\pi_0}\Big[\frac{\eta_{J,t-1}-\eta_{J,t}}{\eta_{J,t-1}}\Big].
\end{multline}
Then, combining the inequalities \eqref{eq:rec1} and \eqref{eq:recurs} recursively for $t=n,\ldots,2$ and then \eqref{eq:rec1} for $t=1$ we obtain
$$
\E_{\pi_0}\Big[\exp\Big(-\eta_{J,n}2^{-1}\sum_{t=1}^{n+1}(\ell_{J,t}(1+ \eta_{J,t-1}\ell_{J,t})+E_{J,t}1_{\eta_{J,t-1} \ell_{J,t}>1/2})\Big)\Big]\
\le  1+\sum_{t=2}^n\E_{\pi_0}\Big[\frac{\eta_{J,t-1}-\eta_{J,t}}{\eta_{J,t-1}}\Big].
$$
We apply the variational form of the entropy \eqref{eq:vf} in order to derive that
\begin{multline*}
0\le \E_{\pi'}\Big[\eta_{J,n}2^{-1}\sum_{t=1}^{n+1}(\ell_{J,t}(1+ \eta_{J,t-1}\ell_{J,t})+E_{J,t}1_{\eta_{J,t-1} \ell_{J,t}>1/2})\Big]\\+\log\Big(1+\sum_{t=2}^n\E_{\pi_0}\Big[\frac{\eta_{J,t-1}-\eta_{J,t}}{\eta_{J,t-1}}\Big]\Big)+\mathcal K(\pi',\pi_0)
\end{multline*}
for any probability measure $\pi'$ on $\{1,\ldots,M\}$. We bound the last term $\mathcal K(\pi',\pi_0)\le \E_{\pi'}[\log(\pi_{J,0}^{-1})]$. By comparing with  the integral of $1/x$ on the interval $[\eta_{j,n},\eta_{j,1}]$, we estimate
$$
\sum_{t=2}^n\frac{\eta_{j,t-1}-\eta_{j,t}}{\eta_{j,t-1}}\le\sum_{t=2}^n \int_{\eta_{j,t}}^{\eta_{j,t-1}} \frac{dx}x\le\int_{\eta_{j,n}}^{\eta_{j,1}} \frac{dx}x\le  \log\Big(\frac{\eta_{j,1}}{\eta_{j,n}}\Big).
$$
We have the  bounds $ \sum_{t=1}^{n+1}E_{j,t}1_{\eta_{j,t-1} \ell_{j,t}>1/2} \le 8E_j$ and $\log(\eta_{j,1}/\eta_{j,n})\le \log(\sqrt n E/2^{-c})$, $1\le j\le M$. Then Theorem \ref{th:regad} is proved    using similar arguments than in the proof of Theorem 6 in \cite{cesabianchi:mansour:stoltz:2007}, choosing $\pi'_j=\eta_{j,n}^{-1}\pi_j/\E_\pi[\eta_{j,n}^{-1}]$ and using the  "gradient trick" as in the proof of Theorem \ref{th:reg}.\end{proof}

The advantage of the adaptive BOA procedure compared with the procedures studied in  \cite{gaillard:stoltz:vanerven:2014,luo:schapire:2015,koolen:vanerven:2015}  is to be adaptive to unknown ranges. The price to pay is an additional logarithmic term $\log(E)+c\log(2)$ depending on the variability of the adaptive learning rates $\eta_{j,t}$ through time. Such losses are avoidable in the case of one single adaptive learning rate $\eta_{j,t}=\eta_{t}$, for all $1\le j\le M$. Notice also that the relative entropy bound is only achieved in the case of one single adaptive learning rate as then $\mathcal K(\pi',\pi_0)=\mathcal K(\pi,\pi_0)$. It is a drawback of the multiple learning rates procedures compared with the single ones of \cite{luo:schapire:2015,koolen:vanerven:2015} achieving such relative entropy bounds. Whether those drawbacks of multiple learning rates procedures can be avoided  is an open question.

\section{Optimality of the BOA procedure in a stochastic environment}\label{sec:se}

\subsection{An empirical stochastic conversion}

We now turn to a stochastic setting where $(X_t,Y_t)$ are random elements observed recursively for $1\le t\le n+1$ ($(X_0,Y_0)=(x_0,y_0)$ arbitrary are considered deterministic). Thanks to the empirical Bernstein inequality of Theorem \ref{th:mart}, the cumulative predictive risk is bounded in term of the regret and a second order term. This new stochastic conversion is provided in Theorem \ref{th:cr} below. The main motivation of the introduction  of the BOA procedure is the following reasoning: as a second order term appears necessarily in the stochastic conversion, an online procedure regularized  by a similar second order  term has nice properties in any stochastic environment. The BOA procedure achieves this strategy as the second order term of the regret bound is similar to the one appearing in the stochastic conversion. Let us go back for a moment to the most general case with no convex assumption on the loss and the notation:
$$
\ell_{j,t}=\ell(X_t,f_{j,t}(X_{t-1}))-\E_{\pi_{t-1}}[\ell(X_t,f_{J,t}(X_{t}))]
$$
for some online aggregation procedure $(\pi_t)_{0\le t\le n}$, i.e. $\pi_t$ is $\sigma(\mathcal D_t)$-measurable. Assume the existence of non increasing sequences $(\eta_{j,t})_t$  that are adapted to $(\mathcal D_t)$ for each $1\le j\le M$ and that satisfy
\begin{equation}\label{eq:bc}
\max_{1\le t\le n+1}\max_{1\le j\le M}\eta_{j,t-1} {\ell_{j,t}}\le 1/2,\qquad a.s..
\end{equation}
We have the following general stochastic conversion that is also valid in the convex case with the associated linearized expression of  $\ell_{j,t}$. It can be seen as an empirical counterpart of the online to batch conversion provided in \cite{zhang:2005,kakade:tewari:2008,gaillard:stoltz:vanerven:2014}. Thanks to the use of the cumulative predictive risk, the conversion holds in a completely general stochastic context; there is no condition on the dependence of the stochastic environment.  Recall that $\E_{\hat \pi}[\textnormal{Err}_{n+1}( f_J)]=\sum_{t=1}^{n+1}\E_{\hat \pi_{t-1}}[\ell(Y_{t} f_{J,t}(X_{t}))]$.
\begin{thm}\label{th:cr}
Under \eqref{eq:bc}, the cumulative predictive risk of any aggregation procedure satisfies, with probability $1-e^{-x}$, $x>0$, for any $1\le j \le M$:
\begin{multline*}
\E_{\hat \pi}[R_{n+1}(f_j)]-R_{n+1}(f_j)\\\le  \E_{\hat \pi}[\textnormal{Err}_{n+1}(f_J)]- \textnormal{Err}_{n+1}(f_j)+\sum_{t=1}^{n+1} \eta_{j,t-1}\ell_{j,t}^2
+ \frac{ \log\Big(1+ \E \Big[\log\Big(\frac{\eta_{j,1}}{\eta_{j,n}}\Big)\Big]\Big)+x}{\eta_{j,n}}.
\end{multline*}
\end{thm}
\begin{proof}
We first note that for each $1\le j\le M$ the sequence $(M_{j,t})_t$ with  $M_{j,t}=\eta(\E_{\hat \pi}[R_t(f_J)]-R_{t}(f_j)-(\E_{\hat \pi}[\textnormal{Err}_{t}(f_j )]- \textnormal{Err}_{t}(f_j)))$ is a martingale adapted to the filtration $(\mathcal D_t)$. Its difference is equal to $\Delta M_{j,t}=\eta(\E_{t-1}[\ell_{j,t}]-\ell_{j,t})$. Then the proof will follow from the classical recursive argument for supermartingales applied to the exponential inequality of Theorem \ref{th:mart}. However, as the learning rates $\eta_{j,t}$ are not necessarily constant, we adapt the recursive argument as in the proof of Theorem \ref{th:regad}.

For any $1\le j\le M$, $1\le t\le n+1$, denoting $X=-\eta_{j,t-1}\ell_{j,t}$ we check that $X\ge -1/2$. We can apply \eqref{eq:cms}  conditionally on $\mathcal D_{t-1}$  and we obtain
$$
\E_{t-1}[\exp(-\eta_{j,t-1}(\ell_{j,t}-\E_{t-1}[\ell_{j,t}])-\eta_{j,t-1}^2\ell_{j,t}^2)]\le 1.
$$
Here we used the fact that $\eta_{j,t-1}$ is $\mathcal D_{t-1}$-measurable. Then we have
\begin{multline*}
\E\Big[\exp\Big(-\eta_{j,t-1}\Big(\sum_{s=1}^t(\ell_{j,s}-\E_{s-1}[\ell_{j,s}])-\eta_{j,s-1}\ell_{j,s}^2\Big)\Big)\Big]\\
\le \E\Big[\exp\Big(-\eta_{j,t-1}\Big(\sum_{s=1}^{t-1}(\ell_{j,s}-\E_{s-1}[\ell_{j,s}])-\eta_{j,s-1}\ell_{j,s}^2\Big)\Big)\Big].
\end{multline*}
To apply the recursive argument we use the basic inequality $x\le x^{\alpha}+(\alpha-1)/\alpha$ for $\alpha=\eta_{j,t-2}/\eta_{j,t-1}\ge 1$ and 
$$
x=\exp\Big(-\eta_{j,t-1}\Big(\sum_{s=1}^{t-1}(\ell_{j,s}-\E_{s-1}[\ell_{j,s}])-\eta_{j,s-1}\ell_{j,s}^2\Big)\Big).
$$
We obtain 
\begin{multline*}
\E\Big[\exp\Big(-\eta_{j,t-1}\Big(\sum_{s=1}^t(\ell_{j,s}-\E_{s-1}[\ell_{j,s}])-\eta_{j,s-1}\ell_{j,s}^2\Big)\Big)\Big]\\
\le \E\Big[\exp\Big(-\eta_{j,t-2}\Big(\sum_{s=1}^{t-1}(\ell_{j,s}-\E_{s-1}[\ell_{j,s}])-\eta_{j,s-1}\ell_{j,s}^2\Big)\Big)\Big]+\E\Big[\frac{\eta_{j,t-2}-\eta_{j,t-1}}{\eta_{j,t-2}}\Big].
\end{multline*}
The same recursive argument than in  the proof of Theorem \ref{th:regad} is applied; we get
$$
\E\Big[\exp\Big(-\eta_{j,n}\Big(\sum_{t=1}^n(\ell_{j,t}-\E_{t-1}[\ell_{j,t}])-\eta_{j,s-1}\ell_{j,s}^2\Big)\Big)\Big]\le 1+\E\Big[\log\Big(\frac{\eta_{j,1}}{\eta_{j,n}}\Big)\Big].
$$
We end the proof by an application of the Chernoff bound.
\end{proof}

\subsection{Second order bounds on the excess of the cumulative predictive risk}
Using  the stochastic conversion of Theorem \ref{th:cr}, we derive from the regret bounds of Section \ref{sec:sob} second order bounds on the cumulative predictive risk of the BOA procedure. As an example, using the second order regret bound of Theorem \ref{th:reg} and the stochastic conversion of Theorem \ref{th:cr} we obtain
\begin{thm}\label{th:regcr}
Assume that the non adaptive BOA  procedure described in Figure \ref{fig:1} is such that $\eta=\eta_{j,t} $ for $1\le j\le M$, $0\le t\le n$ satisfies
condition \eqref{eq:eta}. The BOA  procedure with $
\ell_{j,t}= \ell(Y_{t},f_{j,t-1}(X_t))-\ell(Y_{t},\hat f_{t-1}(X_{t}))
$ has its cumulative predictive risk that satisfies, with probability $1-e^{-x}$, $x>0$:
$$
\E_{\hat \pi}[ R_{n+1}( f_J)]\le  \inf_{\pi} \Big\{\E_\pi\Big[ R_{n+1}(f_j)+2\eta\sum_{t=1}^{n+1} \ell_{J,t}^2\Big]+ \frac{\mathcal K(\pi,\pi_0)+x}{\eta} \Big\}.
$$
Moreover, if $\ell$ is convex with respect to its second argument,  the BOA  procedure with  $
\ell_{j,t}= \ell'(Y_{t},\hat f_{t-1}(X_{t}))(f_{j,t-1}(X_t)- \hat f_{t-1}(X_{t}))
$  has its cumulative predictive risk that also satisfies, with probability $1-e^{-x}$, $x>0$:
$$
 R_{n+1}(\hat f)\le  \inf_{\pi} \Big\{ R_{n+1}(f_\pi)+2\eta\sum_{t=1}^{n+1} \E_\pi[ \ell_{J,t}^2]+ \frac{\mathcal K(\pi,\pi_0)+x}\eta\Big\}.
$$
\end{thm}
\begin{proof}
We prove the result by integrating the result of Theorem \ref{th:cr} with respect to any deterministic $\pi$ and noticing that, as  the learning rates $\eta_{j,t}=\eta$ are constant in the BOA procedure described in figure \ref{fig:1}, 
$ \log (\eta_{j,1}/\eta_{j,n})=0.$
\end{proof}
The main advantage of the new stochastic conversion compared with the one of \cite{gaillard:stoltz:vanerven:2014} is that the empirical second order bound of of the stochastic conversion is similar to the one of the regret bound. We can extend  Theorem \ref{th:regad}; as   the boundedness condition \eqref{eq:bc} is no longer satisfied for any $1\le t\le n+1$, Theorem \ref{th:cr} does not apply directly. We still have
\begin{thm}\label{th:riskEt}
Under the hypothesis of Theorem \ref{th:regad}, the cumulative predictive  risk of the adaptive BOA procedure described in Figure \ref{fig:2} satisfies with probability $1-e^{-x}$, $x>0$,
\begin{multline*}
R_{n+1}(\hat f)\le  \inf_{\pi} \left\{R_{n+1}(f_\pi)+2\E_\pi \left[\sqrt{\displaystyle \sum_{t=1}^{n+1}  \ell_{J,t}^2}\left(\frac{2\sqrt{2}}{\sqrt 2-1}\sqrt{ \log (\pi_{J,0}^{-1})}+\frac{2B_{n,E}+x}{\sqrt{\log(\pi_{J,0}^{-1})}}\right)\right.\right.\\ +2E_J(2(\log(\pi_{J,0}^{-1})+x+2B_{n,E})+9)]\},,
\end{multline*}
where   $B_{n,E}= \log(1+2^{-1}\log(n)+\log(E)+c\log(2) )$ for all $n\ge 1$.\end{thm}
\begin{proof}
As the boundedness condition \eqref{eq:bc} is not satisfied for all $1\le t\le n+1$, we cannot apply directly the result of Theorem \ref{th:cr}. However, one can adapt the proof of the Theorem \ref{th:cr} as we adapted the proof of Theorem \ref{th:cr} for proving Theorem \ref{th:regad}.  As $-\eta_{j,t}(\ell_{j,t+1}-\E_{t-1}[\ell_{j,t+1}])\le \eta_{j,t}E_{j,t+1}$, $1\le j\le M$, we can apply the inequality \eqref{eq:cmsu}
$$
\E_{t-1}\Big[\exp\Big(-\eta_{j,t}2^{-1}(\ell_{j,t+1}-\E_{t-1}[\ell_{j,t+1}])-\eta_{j,t}E_{j,t+1}1_{\eta_{j,t}{\ell_{j,t+1}}> 1/2} \Big)\Big]\le 1.
$$
Thus the proof ends by an application of the same recursive argument as in the proof of Theorem \ref{th:regad}. \end{proof}
For uniform initial weights $\pi_{j,0}=M^{-1}$, the second order bound becomes
\begin{multline}\label{eq:sob}
 R_{n+1}(\hat f)\le  \inf_{\pi}\left\{  R_{n+1}(f_\pi)+C\, \E_\pi\left[\sqrt{ \sum_{t=1}^{n+1}  \ell_{J,t}^2\log(M)}  +E_j(\log(M)+x) \right] \right\}
\end{multline}
for some "constant" $C>0$ increaing as $\log\log n$.
The second order term $\sum_{t=1}^n\ell_{j,t}^2$ is a natural candidate to assert the complexity of Problem (C); the more the $\sum_{t=1}^n\ell_{j,t}^2$ for $1\le j\le M$ are uniformly small and the more one can aggregate the elements of the dictionary optimally. Moreover, this complexity term is observable and it would be interesting to develop a parsimonious strategy that would only aggregate the elements of the dictionary with small complexity terms $\sum_{t=1}^n\ell_{j,t}^2$. Reducing also the size $M$ of the dictionary,  the second order bound \eqref{eq:sob}  can be reduced at the price to decrease the generality of Problem (C), i.e. the number of learners in $\mathcal H$.

The upper bound in  \eqref{eq:sob} is an observable bound for Problem (C) similar than those arising from the PAC bayesian approach, see \cite{catoni:2004} for a detailed study of such empirical bounds in the iid context. It would be interesting to know whether the complexity terms $ \sum_{t=1}^n\ell_{j,t}^2$ are optimal. We are not aware of empirical lower bounds for Problem (C). The bounds developed by \cite{nemirovski:2000,tsybakov:2003,rigollet:2012} are deterministic. To assert the optimality of BOA, it is easy to turn from an empirical bound to a deterministic one. In the iid context, $(X_t,Y_t)$ are iid copies of $(X,Y)$ and the learners are assumed to be constant $f_j=f_{j,t}$, $t\ge 0$, $1\le j\le M$. We then have $
(n+1)^{-1}R_{n+1}(f_j)=\bar R( f_j)=\E[\ell(Y,f_j(X))]$. It is always preferable to convert any online learner $\hat f$ to a batch learner by averaging 
$$
\bar f=\frac1{n+1}\sum_{t=0}^n\hat f_t
$$
as an application of the Jensen inequality gives 
$
\bar R(\bar f)\le (n+1)^{-1}R_{n+1}(\hat f)
$. As $ \ell_{j,t}^2\le  E^2$, Equation \eqref{eq:sob} implies that the batch version of BOA satisfies, with high probability in the iid setting,
$$
\bar R(\bar f) \le  \inf_{\pi}  \bar R (f_\pi) + C\,E\left\{\sqrt{ \frac{\log(M)}{n+1}}+\frac{\log(M)+x}{n+1}\right\}.
$$
Then the BOA procedure is optimal for Problem (C) in the sense of the Definition \ref{def}: in the iid context, the excess of risk of the batch version of the BOA procedure is of order $\sqrt{\log(M)/n}$. The estimate $ \ell_{j,t}^2\le  E^2$ is  very crude and  the complexity terms $\sum_{t=1}^n\ell_{j,t}^2$ can actually be smaller and close to a variance estimate, especially for losses that are similar to the quadratic loss, see the next Section.

\subsection{Optimal learning for Problem (MS)}

The BOA procedure is optimal for Problem (C) and the optimal rate of convergence is also valid in any stochastic environment on the excess of mean predictive risk. To turn to  Problem (MS), we restrict our study to the context of Lipschitz strongly convex losses with iid observations.   Remind that from \cite{tsybakov:2003,rigollet:2012} the optimal rate for Problem (MS) is a fast rate of convergence $\log(M)/n$.
Such fast rates cannot be obtained without regularity assumptions on the loss $\ell$ that force it to behave locally like the square loss, see for instance \cite{audibert:2009}. In the sequel $\ell:\R^2\to\R$ is a loss function satisfying the assumption called {\bf (LIST)} after \cite{kakade:tewari:2008}
\begin{description}
\item[(LIST)] the loss function $\ell$ is $C_\ell$-strongly convex and $C_b$-Lipschitz continuous in its second coordinate  on a convex set $\mathcal C\subset  \R$. 
\end{description}
Recall that a function $g$ is $c$ strongly convex on $\mathcal C\subset \R$ if there exists a constant $c>0$ such that 
$$
g(\alpha a+(1-\alpha)a')
\le \alpha g(a)+(1-\alpha)g(a')-\frac c2 \alpha(1-\alpha)(a-a')^2,$$
for any $a,a'$ $\in \mathcal C$, $0<\alpha<1$.  Under the condition {\bf (LIST)}, few algorithms are known to be optimal in deviation, see \cite{audibert:2007,lecue:mendelson:2009,lecue:rigollet:2013}.\\ 

Note that Assumption {\bf (LIST)} is restrictive and can hold only locally; on a compact set $\mathcal C$, the minimizer $  f(y)^\ast$ of $f(y)\in\R\to \ell (y,f(y))$ exists and satisfies, by strong convexity, $$
\ell(y,f(y))\ge \ell(y, f(y)^\ast)+\frac{C_\ell}2(f(y)-f(y)^\ast)^2.
$$
Moreover, by Lipschitz continuity, $\ell(y,f(y))\le\ell(y,  f(y)^\ast)+C_b|f(y)-f(y)^\ast|$. Thus, necessarily the diameter $D$ of $\mathcal C$ is finite and satisfies  ${C_\ell} D\le 2 C_b$. Then we deduce that $
|\ell_{j,t}|\le C_bD
$, $1\le t\le n+1$, $1\le j\le M$, and under {\bf (LIST)} the ranges are estimated by $E= C_bD$.\\

We obtain the optimality of the BOA procedure for Problem (MS). The result extends easily (with different constants) to any online procedures achieving second order regret bounds on the linearized loss similar to BOA such as the procedures described in \cite{gaillard:stoltz:vanerven:2014,luo:schapire:2015,koolen:vanerven:2015}.
\begin{thm}\label{th:opt}
In the iid setting, under the condition {\bf (LIST)}, for the uniform initial weights $\pi_{j,0}=M^{-1}$, $1\le j\le M$, and for the learning rate $\eta$ satisfying
\begin{equation}\label{eq:cg}
16(e-1)C_b^2/C_\ell  \le \eta^{-1},
\end{equation}
 the cumulative predictive risk of the BOA procedure described in Figure \ref{fig:1} and the risk of its batch version satisfy, with probability $1-2e^{-x}$, 
$$
\bar R(\bar f)+\frac{C_\ell}{2(n+1)}\sum_{t=0}^n\E[(\hat f_t(X)-\bar f(X))^2]\le \frac{R_{n+1}(\hat f)}{n+1}\le \min_{1\le j\le M}\bar R(f_j)+\frac{2 \log(M)+3x}{\eta(n+1)}.
$$
\end{thm}
\begin{proof}
Inspired by the Q-aggregation procedures of \cite{lecue:rigollet:2013}, we start the proof by adding the two second order empirical bounds obtained in Theorem \ref{th:regcr} (using that $\mathcal K(\pi,\pi_0)\le \log(M)$):
\begin{equation}\label{eq:ub}
R_{n+1}(\hat f)+\E_{\hat\pi}[R_{n+1}(f_J)]\le  \inf_{\pi} \Big\{ R_{n+1}(f_\pi)+\E_{\pi}[R_{n+1}(f_J)]\\+4 \eta\sum_{t=1}^{n+1} \E_\pi[  \ell_{J,t}^2] + 2\frac{x+ \log(M)}\eta\Big\}.
\end{equation}
Then we  convert the empirical second order term into a deterministic one. From the "poissonnian" inequality of Lemma A3 of \cite{cesabianchi:lugosi:2006}, as   $0\le   \ell_{j,t}^2\le  C_b^2D^2 \le 1$ under \eqref{eq:cg}, we have
$$
\E_{t-1}[\exp( (\ell_{j,t}/C_bD)^2-(e-1) \E_{t-1}[(\ell_{j,t}/C_bD)^2])]\le 1.
$$
Applying a recursive argument, we show that with probability $1-e^{-x}$
$$
\eta\sum_{t=1}^{n+1} \E_\pi[  \ell_{J,t}^2]\le \eta \sum_{t=1}^{n+1} \E_\pi[ \E_{t-1}[\ell_{J,t}^2]]+\eta (C_bD)^2x.
$$
Using that $\eta C_bD\le (C_\ell D)/(16(e-1)C_b)\le1/2 $ and an union bound, we obtain the deterministic version of the second order bound \eqref{eq:ub}: with probability $1-2e^{-x}$, $x>0$,
\begin{multline*}
R_{n+1}(\hat f)+\E_{\hat\pi}[R_{n+1}(f_J)]\le  \inf_{\pi} \Big\{ R_{n+1}(f_\pi)+\E_{\pi}[R_{n+1}(f_J)]\\+4(e-1)\eta\sum_{t=1}^{n+1} \E_\pi[ \E_{t-1}[\ell_{J,t}^2]]+ \frac{3x+2\log(M)}\eta\Big\}.
\end{multline*}
The optimal fast rate is achieved thanks to a careful analysis of the second order deterministic bound. From the Lipschitz property, the sub-gradient $\ell'$ is bounded by $C_b$ and
\begin{align*}
\E_\pi[ \E_{t-1}[\ell_{J,t}^2]]&\le  C_b^2\E_\pi[ \E_{t-1}[ (f_J(X_{t-1})-\hat f_{t-1}(X_{t-1}))^2]]\\
&\le  C_b^2( V(\pi)+\E[(f_\pi(X)-\hat f_{t-1}(X))^2]),
\end{align*}
where $V(\pi)= \E_\pi[\E[( f_{J}(X)-f_{\pi}(X))^2]]$. As $R_{n+1}(f_\pi)=(n+1)\bar R(f_\pi)$, $\E_{\pi}[R_{n+1}(f_J)]=(n+1) \E_\pi[\bar R(f_J)]$ and combining those bounds we obtain
\begin{multline}\label{eq1}
\frac{R_{n+1}(\hat f)}{n+1}+\frac{\E_{\hat\pi}[R_{n+1}(f_J)]}{n+1}\le  \inf_{\pi} \Big\{\bar R(f_\pi)+\E_{\pi}[\bar R(f_J)] \\+\gamma\Big(V(\pi)+\frac1{n+1}\sum_{t=0}^n\E[(f_\pi(X)-\hat f_{t}(X))^2]\Big)+ \frac{3x+ 2\log(M)}{\eta(n+1)}\Big\}
\end{multline}
with $\gamma=4C_b^2(e-1)\eta$. The rest of the proof is inspired by the reasoning of \cite{lecue:rigollet:2013}. First, one can check the identity
$$
V(\pi)-V(\pi')=<\nabla V(\pi'),(\pi-\pi')>-\E[(f_\pi(X)-f_{\pi'}(X))^2]
$$
where $\pi$ and $\pi'$ are any weights vectors and $<\cdot,\cdot>$ denotes the scalar product on $\R^M$. By $C_\ell$-strong convexity one can also check that
$$
\bar R(f_\pi)-\bar R(f_{\pi'})\ge <\nabla \bar R(f_{\pi'}),(\pi-\pi')>+\frac{C_\ell}{2}\E[(f_\pi(X)-f_{\pi'}(X))^2].
$$
Thus the function $H$: $\pi\to \bar R(f_\pi)+\E_\pi[\bar R(f_j)]+\gamma V(\pi)$ is convex as $0\le \gamma\le C_\ell/2$ under \eqref{eq:cg}. Moreover, if one denotes $\pi^\ast$ a minimizer of $H$, we have for any weights $\pi$
$$
H(\pi)-H(\pi^\ast) \ge \Big(\frac{C_\ell}{2}-\gamma \Big) \E[(f_\pi(X)-f_{\pi^\ast}(X))^2].
$$ 
Thus, applying this inequality to $\hat \pi$ we obtain
\begin{multline*}
\frac{C_\ell/2-\gamma}{n+1}\sum_{t=0}^n\E[(f_{\pi^\ast}(X)-\hat f_{t}(X))^2]\le \frac{R_{n+1}(\hat f)}{n+1}+\frac{\E_{\hat \pi}[R_{n+1}(f_J)]}{n+1}\\-\Big(\bar R(f_{\pi^\ast})+\E_{\pi^\ast}[\bar R(f_J)]+\gamma V(\pi^\ast)\Big) +\frac\gamma{n+1} \sum_{t=0}^nV(\pi_t).
\end{multline*}
Combining this last inequality with the inequality \eqref{eq1} we derive that
$$
\frac{C_\ell/2-\gamma}{n+1}\sum_{t=0}^n\E[(f_{\pi^\ast}(X)-\hat f_{t}(X))^2]\le \frac{3x+ 2\log(M)}{\eta(n+1)}+\frac{ \gamma}{ n+1} \sum_{t=0}^nV(\pi_t).
$$
Plugging in this new estimate into \eqref{eq1} we obtain
\begin{multline*}
\frac{R_{n+1}(\hat f)}{n+1}+\frac{\E_{\hat\pi}[R_{n+1}(f_J)]}{n+1}-\frac{2\gamma^2}{C_\ell-2\gamma}\frac1{n+1} \sum_{t=0}^nV(\pi_t)\le  \bar R(f_{\pi^\ast})+\E_{\pi^\ast}[\bar R(f_J)]+\gamma V(\pi^\ast)\\+\frac{C_\ell}{C_\ell -2\gamma}\frac{3x+2 \log(M)}{\eta(n+1)}.
\end{multline*}
Now, using $C_\ell$-strong convexity as in Proposition 2 of \cite{lecue:rigollet:2013}, we have for any probability measure $\pi$
\begin{equation}\label{eq:sca}
\bar R(f_\pi)\le \E_\pi[\bar R(f_J)]-\frac{C_\ell V(\pi)}{2}.
\end{equation}
As under condition \eqref{eq:cg} it holds
$$
\frac{2\gamma^2}{C_\ell-2\gamma}\le \frac{C_\ell}{2}\qquad \mbox{and}\qquad\frac{C_\ell }{C_\ell -2\gamma}\le 2,
$$
we can use the strong convexity argument \ref{eq:sca} for any $\pi_t$, $0\le t\le n$ and obtain
$$
2\frac{R_{n+1}(\hat f)}{n+1}\le \bar  R(f_{\pi^\ast})+\E_{\pi^\ast}[\bar R(f_J)]+\gamma V(\pi^\ast)\\+2\frac{3x+ 2\log(M)}{\eta(n+1)}.
$$
The proof ends by noticing that $$
\bar R(f_{\pi^\ast})+\E_{\pi^\ast}[\bar R(f_J)]+\gamma V(\pi^\ast)\le 2\min_{1\le j\le M}\bar R(f_j).
$$
The lower bound on $R_{n+1}(\hat f)/(n+1)$ follows by an application of the strong convexity argument applied to $\bar f=(n+1)^{-1}\sum_{t=0}^n\hat f_t$.
\end{proof}
Theorem \ref{th:opt} provides the optimality of the BOA procedure for Problem (MS) because 
$$
\bar R(\bar f) \le \min_{1\le j\le M}\bar R(f_j)+\frac{ 2\log(M)+3x}{\eta(n+1)}.
$$
The additional term 
$$
\frac{C_\ell}{2(n+1)}\sum_{t=0}^n\E[(\hat f_t(X)-\bar f(X))^2]
$$
is the benefit of considering the batch version of BOA under the strong convexity assumptions {\bf (LIST)}. As the fast rate is optimal, the partial sums $\sum_{t=0}^n\E[(\hat f_t(X)-\bar f(X))^2]$ might converge to a small constant. Assuming that $\bar f$ is converging with $n$, the convergence of the partial sums implies that $\E[(\hat f_n(X)-\bar f(X))^2]=o(n^{-1})$. Thanks to the Lipschitz assumption on the loss, it implies that the difference $|\bar R(\bar f)-\bar R(\hat f_n)|\le C_b \E[(\hat f_n(X)-\bar f(X))^2]$ is small. The difference $|\bar R(\bar f)-\bar R(\hat f_n)|$ is then negligible compared with the fast rate $\log(M)/n$. Then, at the price of some constant $C>1$, we also have
$$
\bar R(\hat f_n) \le \min_{1\le j\le M}\bar R(f_j)+C\frac{2 \log(M)+3x}{\eta(n+1)}.
$$
It means that in the iid setting, the predictive risk of the online procedure $\bar R(\hat f_n)=\E[\ell(Y_{n+1},\hat f_n(X_{n+1})~|~\mathcal D_n]$ might be optimal. It would be interesting to check rigorously if it is the case and to extend this result to non iid settings following the reasoning developed in \cite{mohri:rostamizadeh:2010}.

The tuning parameter $\eta$ can be considered as the inverse of the temperature $\beta$ of the $Q$-aggregation procedure studied in \cite{lecue:rigollet:2013}. In the $Q$-aggregation, the tuning parameter $\beta$ is required to be larger than $60 C_b^2/C_\ell$. It is a condition similar than our restriction \eqref{eq:cg} on $\eta$. The larger is $\eta$ satisfying the condition \eqref{eq:cg} and the best is the rate of convergence. The choice $\eta^\ast =(16(e-1)C_b^2/  C_\ell)^{-1} $ is optimal. The resulting BOA procedure is non adaptive in the sense that it depends on   the range $C_b$ of the gradients that can be unknown. On the contrary, the multiple learning rates BOA procedure achieves to tune automatically the learning rates. At the price of larger "constants" that grow as $\log\log(n)$, we extend the preceding optimal rate of convergence to the adaptive BOA procedure:

\begin{thm}\label{th:list}
In the iid setting, under the condition {\bf (LIST)}, for the uniform initial weights $\pi_{j,0}=M^{-1}$, $1\le j\le M$, the mean predictive risk of the adaptive BOA procedure described in Figure \ref{fig:2}   and the risk of its batch version satisfy, with probability $1-2e^{-x}$, 
\begin{multline*}
\bar R(\bar f)+\frac{C_\ell}{2(n+1)}\sum_{t=0}^n\E[(\hat f_t(X)-\bar f(X))^2]\\\le\frac{ R_{n+1}(\hat f)}{n+1}
\le \min_{1\le j\le M}\bar R(f_j)+  \frac{C_b^2}{C_\ell}\frac{668   \log (M)   +55\frac{(B_{n,C_bD}+x/2)^2}{\log(M)}+4(9+2( x+2B_{n,C_bD}))}{n+1},
\end{multline*}
where  $B_{n,C_bD}= \log(1 +2^{-1}\log(n)+\log(C_bD) +c\log(2))$ for all $n\ge1$.
\end{thm}
\begin{proof} The proof starts from the second order empirical bound provided in Theorem \ref{th:riskEt} in the iid context under {(\bf LIST)}, where $|\ell_{j,t}|\le C_bD$, used with and without the gradient trick; next, from the Young inequality, we have for any $\eta>0$
\begin{multline*}
R_{n+1}(\hat f)+\E_{\hat\pi}[R_{n+1}(f_j)]\le  \inf_{\pi} \Big\{\bar R(f_\pi)+\E_{\pi}[R(\bar f_j)]+4\eta \E_\pi \Big[\sum_{t=1}^{n+1}  \ell_{j,t}^2\Big]\Big\}\\ +2\frac{ 12\log(M)+(B_{n,C_bD}+x/2)^2/\log(M)}\eta+2C_bD(2(\log (M)+ x+ 2B_{n,C_bD})+9),
\end{multline*}
using that   $2\sqrt 2/(\sqrt 2-1))\le \sqrt{48}$ and $B_{n,E_j}\le B_{n,C_bD}$. Then we can use the "poissonnian" inequality as in the proof of Theorem \ref{th:opt} to obtain the deterministic second order bound, with $\gamma=4C_b^2(e-1)\eta$,
\begin{multline*}
 \frac{ R_{n+1}(\hat f)}{n+1}+\frac{\E_{\hat\pi}[R_{n+1}(f_j)]}{n+1}\le  \inf_{\pi} \Big\{\bar R(f_\pi)+\E_{\pi}[\bar R(f_j)]+\frac{\gamma}{n+1} \E_\pi \Big[\sum_{t=1}^{n+1}  \E_{t-1}[\ell_{j,t}^2]\Big]\Big\}\\ +2\frac{ 12\log(M)+(B_{n,C_bD}+x/2)^2/\log(M)}{\eta(n+1)}+\frac{2C_bD(2(\log (M)+ x+ 2B_{n,C_bD})+9)}{n+1}.
\end{multline*}
The proof ends similarly than the one of Theorem \ref{th:opt}. For $\eta^\ast$ satisfying the equality in the condition \eqref{eq:cg}, we obtain
\begin{multline*}
 \frac{ R_{n+1}(\hat f)}{n+1}\le \min_{1\le j\le M}\bar R(f_j)+2\frac{ 12\log(M)+(B_{n,C_bD}+x/2)^2/\log(M)}{\eta^\ast(n+1)}\\+\frac{2C_bD(2(\log (M)+ x+ 2B_{n,C_bD})+9)}{n+1}.
\end{multline*}
The result follows from the expression of $\eta^\ast$, the strong convexity of the risk and the estimate $C_bD\le 2C_b^2/C_\ell$.\end{proof}

The BOA procedure is explicitly computed with complexity $O(Mn)$. It is a practical advantage  compared with the batch procedures studied in \cite{audibert:2007,lecue:mendelson:2009,lecue:rigollet:2013} that require a computational optimization technique. This issue has been solved in \cite{dai:rigollet:zhang:2012} for the square loss using greedy iterative algorithms that approximate the $Q$-aggregation procedure.  
\section*{Acknowledgments}
{I am grateful to two anonymous referees for their helpful comments. I would also like to thank Pierre Gaillard and Gilles Stoltz for valuable comments on a preliminary version.}
 
\bibliographystyle{amsalpha}
\bibliography{Bernstein}

\end{document}